\documentclass{article}
\usepackage{ijcai26}

\usepackage{times}
\usepackage{soul}
\usepackage{url}
\usepackage[hidelinks]{hyperref}
\usepackage[utf8]{inputenc}
\usepackage[small]{caption}
\usepackage{graphicx}
\usepackage{amsmath}
\usepackage{amsthm}
\usepackage{booktabs}
\usepackage{algorithm}
\usepackage{algorithmic}
\usepackage[switch]{lineno}

\newtheorem{example}{Example}
\newtheorem{theorem}{Theorem}

\newtheorem{definition}{Definition}
\usepackage{amssymb} 

\title{Reasoning and Planning with Dynamically Changing Norms}

\author{
Taylor Olson$^1$
\and
Roberto Salas-Damian\and
Kenneth D. Forbus$^{2}$
\affiliations
$^1$University of Iowa\\
$^2$Northwestern University\\
\emails
taylolson@uiowa.edu,
roberto.salas@u.northwestern.edu,
forbus@northwestern.edu
}

\begin{document}

\maketitle

\begin{abstract}
To safely interact with humans, AI agents must both know our norms and consider them during planning. However, such norm-guided planning has been less explored, only within communities of artificial agents, and has ignored the dynamic nature of norms. This paper instead presents an approach to guiding planning with dynamically changing norms in a human-AI setting. We contribute a defeasible calculus for resolving normative conflicts and an approach to using such dynamically changing norms as guard rails on plans. We theoretically demonstrate our approach with formal proofs and empirically with an AI agent, SocialBot, on a natural language dialogue task.
\end{abstract}

\section{Introduction}

Imagine that Karli says, ``Do not share my medical records.'' Suppose Karli then gets married. So, she says, ``You may tell my husband what prescriptions I'm taking.'' Now, suppose Karli then has children. So she says, ``You must share my health conditions with my children.'' How can an AI agent respect Karli's dynamically changing wishes?

As AI agents become more integrated into our social world, they must be able to learn our norms and adapt their behavior accordingly. But while there have been developments formalizing norm learning and reasoning~\cite{savarimuthu2011norm,sarathy2017learning,olson2021learning,olson2023mitigating}, norm-guided planning has been less explored, only within communities of artificial agents and ignoring the dynamic nature of norms. Other approaches like LLMs may hold sophisticated dialogues at times, but they have proven to be manipulable \cite{zhuo2025bypassing,derner2023beyond} and lack a theoretical foundation for normative reasoning.

This paper presents a theoretically grounded approach to guiding plans with dynamically changing norms in a human-AI setting. We contribute a defeasible calculus for resolving normative conflicts. We theoretically demonstrate this logic with formal proofs. Then we present an approach to using such dynamically changing norms as guard rails on plans.  We empirically evaluate this approach in the domain of privacy with SocialBot, an AI dialogue agent built on top of the Companion cognitive architecture \cite{forbus2017analogy}.

This paper is organized as follows. We start by providing background on the formal representations we draw upon. Next, we introduce our approach to reasoning about dynamically changing norms and utilizing them during planning. Then we theoretically demonstrate our approach with examples and formal proofs. Next, we describe an experiment with SocialBot on a novel synthetic dataset. We conclude with related work, limitations, and future work.

\section{Background}

Here we consider a \textit{norm} to be a deontic judgment (obligatory, optional, or impermissible) of a behavior that is often context specific. Here we model learning an agent's \textit{normative beliefs} from their \textit{normative testimony}. Normative testimony are natural language statements introducing norms \cite{hills2009moral}. For example, ``do not share my medical records.'' A normative belief is an agent's belief in a norm. For example, Karli believes that \textit{one should not share her medical records with her mother}.  These two concepts cannot be collapsed. Our inference from Karli's statement, ``do not share my medical records'', to her belief that \textit{we should not share her prescriptions with her mother}, is an estimation. She may believe this is fine but just hasn't explicitly stated it as an exception.

\subsection{NextKB Ontology}
We draw upon the predicate calculus language of the NextKB\footnote{https://www.qrg.northwestern.edu/nextkb/index.html} ontology to represent the world. NextKB is derived from OpenCyc (an open-source subset of the Cyc ontology~\cite{lenat1995cyc,matuszek2006introduction}) and contains many concepts, relations, and facts. This knowledge is contextualized using Cyc-style \textit{microtheories} \cite{guha2004contexts}, or hierarchical subsets of the knowledge base. This, for example, enables separately tracking the beliefs of different agents and constraining reasoning to only consider these beliefs. Where \texttt{Mt-1} and \texttt{Mt-2} are microthoeries , the higher-order predicate \texttt{(genlMt Mt-1 Mt-2)} represents that \texttt{Mt-1} inherits all facts from microtheory \texttt{Mt-2} i.e.,  \texttt{Mt-2} $\subseteq$ \texttt{Mt-1}.

Rules in NextKB are represented as Horn clauses. Where \texttt{<HEAD>} and \texttt{<BODY-X>} are logical statements, such rules are of the form below.

\texttt{(<== <HEAD>}

\texttt{\hspace{2mm}<BODY-1>...<BODY-N>)}

This can be read as, ``\texttt{<HEAD>} can be inferred given that \texttt{<BODY-1>...<BODY-N>} can be proven.'' When \texttt{S} is derivable via the axioms and Horn clauses in \texttt{Mt}, we write \texttt{Mt} $\vdash$ \texttt{S}, where \texttt{Mt} is a microtheory and \texttt{S} is logical sentence.

\subsection{Normative Concepts}
For representing normative testimony, we draw upon the formalism of~\cite{olson2023mitigating,olson2021learning}, which utilizes frame representations for both norms and actions (neo-Davidsonian action representations). Thus, it better supports incremental learning than other formalisms e.g.,~\cite{vasconcelos2009normative}. This enables learning via natural language (NL) and is thus better for human-AI settings.

\begin{definition}[Normative Testimony as Norm Frames] A \textit{Norm Frame} is a logical encoding of a norm of the form:

\texttt{(isa <norm> Norm)}

\texttt{(context <norm> <context>)}

\texttt{(behavior <norm> <behavior>)}

\texttt{(evaluation <norm> <deontic>)}

Where \texttt{<norm>} is a constant representing the norm, \texttt{<context>} is a conjunction of first-order literals that must be true for the norm to be \textit{active}, \texttt{<behavior>} is a conjunction of positive first-order literals representing the behavior that the norm \textit{applies to}, and \texttt{<deontic>} is a deontic operator $\in$ \texttt{\{Obligatory, Optional,Impermissible\}} (the three-fold classification of deontic logic \cite{mcnamara1996making}).
\end{definition}

We note that norm frames are true with respect to a particular microtheory, representing a norm of a specific agent. Second, by first-order \textit{literals} we mean atomic sentences consisting of n-ary predicates applied to n terms (variables are prefixed with ```?'''): \texttt{(<pred> <arg-1>...<arg-n>)}. Third, empty conjunctions as behaviors and/or contexts represent tautologies: \texttt{(and)} $\equiv \top$.

To illustrate, Karli's normative testimony ``You may share my prescriptions with my husband'' could be represented as the norm frame below within Karli's microtheory. Note that \texttt{SelfToken} is a self-referential token for the agent.

\texttt{(isa norm1 Norm)}

\texttt{(context norm1}

\texttt{\hspace{2mm}(and (husbandOf Karli ?hubby)))}

\texttt{(behavior norm1}

\texttt{\hspace{2mm}(and (isa ?act RevealingPrescription)}

\texttt{\hspace{4mm}(prescriptionOf ?act Karli)}

\texttt{\hspace{4mm}(recipientOfInfo ?act ?hubby)}

\texttt{\hspace{4mm}(senderOfInfo ?act SelfToken)))}

\texttt{(evaluation norm1 Optional)}

\subsubsection{Norm Conflicts}

Although our normative beliefs admit many exceptions, it would be difficult, if not impossible, to exhaustively state these up front. We thus regularly add exceptions to and override our previous normative testimony. This results in a stream of conflicting norms. Here, we adopt the common three-fold classification of norm conflicts~\cite{ross1958law}. 

\begin{definition}[Ontology of Norm Conflicts] Two norms whose evaluations are inconsistent (e.g., Obligatory and Optional) conflict where their contexts intersect when:
\begin{itemize}
    \item \textit{Direct conflict:} their behaviors are equivalent;
    \item \textit{Indirect conflict:} the behavior of one norm entails, and is not equivalent to, the behavior of the other; 
    \item \textit{Intersecting conflict:} their behaviors intersect but do not entail one another.
\end{itemize}
\label{def:conflict-ontology}
\end{definition}

With this background, next we present our approach to resolving norm conflicts and using such dynamically changing norms to guide plans.


\section{Approach}

We assume an abstract notion of a \textit{plan} that consists of a set of preconditions defining when an ordered list of actions can be executed. Various AI planning formalisms fit into this notion (e.g., we demonstrate with Hierarchical Task Networks (HTNs)~\cite{Nau1999} in later sections). We formalize norms as guard-rails on plans as \textit{norm-guided plans}.

\begin{definition} [Norm-Guided Plan] A norm-guided plan is a plan that first checks the normative beliefs of a relevant agent. Formally, this is a plan with a normative belief \texttt{B} in its set of preconditions.

\texttt{(plan} 

\texttt{\hspace{2mm} (and (<c-1>...B...<c-n>))}

\texttt{\hspace{2mm} (TheList <act-1>...<act-m>))}

\end{definition}

Thus, we consider a limited view of \textit{norm-guided} planning here in which norms serve as guard rails, rather than motivators. Once the agent has decided on a plan that is relevant in the current context, it only executes that plan if it is proven to be permissible. This view of norm-guided planning allows us to reduce obligations and discretionary norms to permissions. For example, Karli stating ``You must share my medical records with my children'' does not create any intention to act in our formalism. The obligation merely makes this act permissible when the relevant plan is proposed. Therefore, all norm frames with evaluation  of \texttt{Obligatory} or \texttt{Optional} are reduced to \texttt{Permissions}. Those with evaluation of \texttt{Impermissible} are called \texttt{Prohibitions}.

Given that plans only execute when all preconditions are true, norm-guided plans can thus ensure that actions will never be executed if a particular normative belief is inferred. We formalize \textit{normative beliefs} below.

\begin{definition}[Normative Beliefs] \textbf{Normative beliefs} are represented as higher-order binary predicates, \texttt{permissible} or \texttt{impermissible}, representing a particular agent's normative belief. Where \texttt{<b>} and \texttt{<c>} are conjunctions of first-order literals, when true in an agent's microtheory, \texttt{(permissible <b> <c>)} holds that the agent believes behavior \texttt{<b>} is permissible in context \texttt{<c>}, and \texttt{(impermissible <b> <c>)} that it is impermissible. As in deontic logic, \texttt{(impermissible <b> <c>)} $\equiv$ \texttt{(not (permissible <b> <c>))}.
\end{definition}

We then formalize a defeasible calculus that infers normative beliefs from ongoing, possibly conflicting normative testimony. This calculus resolves conflicts in stated testimony, ensuring that this inference is sound. In other words, it ensures that, in a given context, a behavior cannot be inferred to be believed to be two inconsistent deontic statuses. For example, if \textit{sharing Karli's medical records with her mother} is inferred to be impermissible, then it cannot also be inferred to be permissible. Formally, where our calculus is denoted as $\mathcal{C}$, given a microtheory $N$ containing normative testimony, if $N \vdash_{\mathcal{C}} permissible(b,c)$, then $N \nvdash_{\mathcal{C}} \neg permissible(b,c)$. Given that normative beliefs serve as guard rails on AI agents' actions, this property ensures their behavior is consistent. We describe our defeasible calculus next.

\subsection{Reasoning With Norms}

Detecting conflicts in normative testimony requires reasoning about how they relate. We formalize this reasoning below.

\begin{definition} [Entailment] Let \texttt{C1} and \texttt{C2} be conjunctions of first-order literals. Given that all free variables in \texttt{C1} and \texttt{C2} are interpreted as universally quantified variables, by Universal Instantiation we can substitute each for a unique arbitrary constant. Let \texttt{C1'} be such an arbitrary Universal Instantiation of \texttt{C1} and \texttt{Mt} be a microtheory. \texttt{Mt} $\vdash$ \texttt{(entails C1 C2)} if and only if \texttt{Mt}, \texttt{C1'}$\vdash$ \texttt{C2}. Thus, \texttt{entails} is a formalization of the Deduction Theorem \cite{franks2021deduction}.
\label{def:entailment}
\end{definition}

To save space, we leave microtheory declarations for truth implicit from here on out unless necessary.

\begin{definition} [Active] A norm frame is \textbf{active} in a situation \texttt{C'} when it entails the norm frame's context. Formally, given norm frame \texttt{N} where \texttt{(context N C)} is true, norm frame \texttt{N} is \textit{active} in \texttt{C'} when \texttt{(entails C' C)} is true.
\end{definition}

\begin{definition}[Application Grounds] The \textbf{application grounds} of a norm is the set of all behaviors in which that norm applies~\cite{elhag2000formal}. Formally, given \texttt{(behavior N B)} is true for norm frame \texttt{N}, if \texttt{(entails B' B)} is true, then \texttt{B'} $\in$ the application grounds of \texttt{N}.
\end{definition}

\begin{definition} [Intersect] Two conjunctions of first-order literals \texttt{C1} and \texttt{C2} \textbf{intersect} at \texttt{C3} when \texttt{(entails C1 C3)} and \texttt{(entails C2 C3)} are both true. They are said to \textbf{strictly intersect} when they intersect and neither \texttt{(entails C1 C2)} nor \texttt{(entails C2 C1)} are true.
\end{definition}

\begin{definition}[Subsume] The application grounds \texttt{B1} of norm \texttt{N1} \textbf{subsume} the application grounds \texttt{B2} of norm \texttt{N2} when \texttt{B1} $\subseteq$ \texttt{B2}, i.e., when \texttt{(entails b1 b2)}, where \texttt{b1} and \texttt{b2} are the behaviors of \texttt{N1} and \texttt{N2} respectively. \texttt{N1}'s application grounds \textbf{strictly subsumes} \texttt{N2}'s when it subsumes it, yet \texttt{(entails b2 b1)} is false. We often use a shorthand here and say that norm frame \texttt{N1} (strictly) subsumes \texttt{N2}.
\end{definition}

\begin{definition} [Temporal Ordering of Norms] A \textbf{temporal ordering of norms} is created as norm frames are created. Each norm frame has a time stamp and \texttt{(normPriorToNorm N1 N2)} holds that \texttt{N1}'s timestamp is before \texttt{N2}'s.
\end{definition}

Next, we use these definitions for norm conflict resolution.

\subsection{Resolving Norm Conflicts}

Our formalism requires making a default assumption about agents' normative beliefs when they have not provided any normative testimony. There are two possible assumptions to make: \textit{Prohibitive Closure}, or actions are impermissible by default, and \textit{Permissive Closure}, or actions are permissible by default (i.e., weak permission \cite{von1963norm}). Choosing between these two assumptions depends on the situation. For acts like sharing certain user information and other more sensitive acts, it is more reasonable to make a Prohibitive Closure assumption. Whereas for simple acts like walking or sitting, a Permissive Closure assumption is more reasonable. In the next sections, we formalize norm conflict resolution under both assumptions with defeasible Horn clause rules. By expanding on the idea of deontic inheritance~\cite{ross1944imperatives}, these inference rules dynamically infer agents' normative beliefs based on their normative testimony. As a reminder, a norm frame is of type \texttt{Permission} when its evaluation is \texttt{Obligatory} or \texttt{Optional} and of type \texttt{Prohibition} when its evaluation is \texttt{Impermissible}. \texttt{uninferredSentence} represents negation as failure. 

\subsubsection{Resolving Norm Conflicts Under Prohibitive Closure}

\begin{definition}[Inference Rule 1] An agent believes a behavior is permissible in a given context when they have stated a permission that is active in that context, the behavior is on its application grounds, and the permission is not defeated.

\texttt{(<== (permissible ?b ?c)}

\texttt{\hspace{2mm}(isa ?perm Permission)}

\texttt{\hspace{2mm}(context ?perm ?c1)}

\texttt{\hspace{2mm}(behavior ?perm ?b1)}

\texttt{\hspace{2mm}(entails ?c ?c1)}

\texttt{\hspace{2mm}(entails ?b ?b1)}

\texttt{\hspace{2mm}(uninferredSentence} 

\texttt{\hspace{4mm}(permissionDefeated ?perm ?b1 ?c1 ?b}

\texttt{\hspace{6mm}?c ?proh)))}
\end{definition}

Permissions are defeated, or \texttt{(permissionDefeated ?perm ?b1 ?c1 ?b 
			?c ?proh)} is true, under two conditions, encoded with the following two Horn clause rules.

\begin{definition}[Exception 1.1] The agent later states a prohibition that is also active in the context, whose application grounds subsumes the permission's.

\texttt{(<== (permissionDefeated ?perm ?b1 ?c1}

\texttt{\hspace{12mm}
				?b ?c ?proh)}
    
\texttt{\hspace{2mm}(isa ?proh Prohibition)}
	
\texttt{\hspace{2mm}(normPriorToNorm ?perm ?proh)}

\texttt{\hspace{2mm}(context ?proh ?c2)}

\texttt{\hspace{2mm}(behavior ?proh ?b2)}

\texttt{\hspace{2mm}(entails ?c ?c2)}
 
\texttt{\hspace{2mm}(entails ?b1 ?b2))}

\end{definition}

\begin{definition}[Exception 1.2] The agent stated a prohibition that is also active in the context, the behavior being evaluated is on the prohibition’s application grounds, and the prohibition’s application grounds do not subsume the permission.

\texttt{(<== (permissionDefeated ?perm ?b1 ?c1}

\texttt{\hspace{12mm}
				?b ?c ?proh)}

\texttt{\hspace{2mm}(isa ?proh Prohibition)}

\texttt{\hspace{2mm}(context ?proh ?c2)}

\texttt{\hspace{2mm}(behavior ?proh ?b2)}

\texttt{\hspace{2mm}(entails ?c ?c2)}
	
\texttt{\hspace{2mm}(entails ?b ?b2)}

\texttt{\hspace{2mm}(uninferredSentence}

\texttt{\hspace{4mm}(entails ?b1 ?b2)))}

\end{definition}

If \texttt{(permissible ?b ?c)} cannot be proven, either through defeat or lack of evidence, then the agent is assumed to believe that the behavior is impermissible in the given context. Thus, the Prohibitive Closure assumption operates as negation as failure, which is formalized as below.

\begin{definition}[Prohibitive Closure] When it cannot be proven that an agent believes a behavior is permissible in a given context, assume they believe it is impermissible.

\texttt{(<== (impermissible ?b ?c)}

\texttt{\hspace{2mm}(uninferredSentence}

\texttt{\hspace{4mm}(permissible ?b ?c)))}
\end{definition}

\subsubsection{Resolving Norm Conflicts Under Permissive Closure}
Below we formalize the same theory of norm conflict resolution under a Permissive Closure assumption.

\begin{definition}[Inference Rule 2] An agent believes a behavior is impermissible in a given context when they have stated a prohibition that is active in that context, the behavior is on its application grounds, and the prohibition is not defeated.

\texttt{(<== (impermissible ?b ?c)}

\texttt{\hspace{2mm}(isa ?proh Prohibition)}

\texttt{\hspace{2mm}(context ?proh ?c1)}

\texttt{\hspace{2mm}(behavior ?proh ?b1)}

\texttt{\hspace{2mm}(entails ?c ?c1)}

\texttt{\hspace{2mm}(entails ?b ?b1)}

\texttt{\hspace{2mm}(uninferredSentence} 

\texttt{\hspace{4mm}(prohibitionDefeated ?proh ?b1 ?c1 ?b}

\texttt{\hspace{6mm}?c ?perm)))}
\end{definition}

Prohibitions are defeated, or \texttt{(prohibitionDefeated ?proh ?b1 ?c1 ?b 
			?c ?perm)} is true, under a single condition, encoded with the following Horn clause rule.

\begin{definition}[Exception 2.1] The agent later states a permission that is also active in the context, whose application grounds are subsumed by the prohibition's.

\texttt{(<== (prohibitionDefeated ?proh ?b1}

\texttt{\hspace{14mm}?c1 ?b ?c ?perm)}
    
\texttt{\hspace{2mm}(isa ?perm Permission)}
	
\texttt{\hspace{2mm}(normPriorToNorm ?proh ?perm)}

\texttt{\hspace{2mm}(context ?perm ?c2)}

\texttt{\hspace{2mm}(behavior ?perm ?b2)}

\texttt{\hspace{2mm}(entails ?c ?c2)}
 
\texttt{\hspace{2mm}(entails ?b2 ?b1))}

\end{definition}

If \texttt{(impermissible ?b ?c)} cannot be proven, either through defeat or lack of evidence, then the agent is assumed to believe the behavior is permissible in the given context. This Permissive Closure assumption is formalized below.

\begin{definition}[Permissive Closure] When it cannot be proven that an agent believes a behavior is impermissible in a given context, assume they believe it is permissible.

\texttt{(<== (permissible ?b ?c)}

\texttt{\hspace{2mm}(uninferredSentence}

\texttt{\hspace{4mm}(impermissible ?b ?c)))}
\end{definition}

In summary, depending on which assumption one makes, Inference Rules 1 and 2 resolve conflicts in an agent's normative testimony to compute their current normative beliefs. As preconditions in norm-guided plans, normative beliefs thus serve as dynamic guard rails for an AI agent's actions.


\section{Theoretical Evaluation}

In this section, we demonstrate norm conflict resolution under all possible conflict types (see Definition \ref{def:conflict-ontology}) with formal proofs. To save space, we do so only under a Prohibitive Closure assumption (Inference Rule 1). However, without loss of generality, all theorems also hold under a Permissive Closure assumption as well (besides when the set of normative testimony is empty). To help illustrate, we also provide examples for each theorem.

\begin{example}
    Karli says, ``You may tell my husband what prescriptions I am taking.''
    Now sharing Karli's prescriptions with her husband is permissible, but it is assumed that sharing any other medical information is still impermissible. 
\end{example}

\begin{theorem}[Direct and Indirect Defeats of Prohibitions by Subsumed Obligations and Discretionary Norms]
\label{direct-ind-imp-obl}
By Inference Rule 1,  prohibitions get exceptions added by later stating subsumed obligations and discretionary norms.
\end{theorem}

\begin{proof}
Due to the Prohibitive Closure assumption, we must prove this for two cases: 1) when an agent has not yet explicitly stated a prohibition, and 2) when they have.

\textit{Case 1: Implicit Prohibition.} Let \texttt{A} be an agent who has never stated a prohibition. Let \texttt{N} be the norm frame encoding of an obligation or discretionary norm stated by \texttt{A}, with behavior \texttt{B} and context \texttt{C}. Let context \texttt{?c} be a context in which the obligation is active, or that 
\texttt{(entails ?c C)} is true. Let  \texttt{?b} be a behavior on the obligation's application grounds, or that \texttt{(entails ?b B)} is true. We show that in context \texttt{?c}, behavior \texttt{?b} is believed to be permissible by agent \texttt{A}.

Given that \texttt{N} is an obligation or discretionary norm, \texttt{N} is also a permission. Given that \texttt{A} has never stated a prohibition, both exceptions to Inference Rule 1 are false, and thus \texttt{(uninferredSentence (permissionDefeated N B C ?b ?c ?proh)))} is true. Thus, by Inference Rule 1, \texttt{(permissible ?b ?c)} is true in \texttt{A}'s microtheory. Therefore, obligations and discretionary norms add exceptions to the Prohibitive Closure assumption.

\textit{Case 2: Explicit Prohibition.} Let \texttt{P} be the norm frame encoding of a prohibition stated by agent \texttt{A}, with behavior \texttt{B1} and context \texttt{C1}. Let \texttt{N} be the norm frame encoding of an obligation or discretionary norm stated later by \texttt{A}, with behavior \texttt{B2} and context \texttt{C2}. Assume the prohibition's application grounds subsume norm \texttt{N}'s, and thus \texttt{(entails B2 B1)} is true. Let context \texttt{?c} be a context in which both norms are active, or that \texttt{(entails ?c C1)} and \texttt{(entails ?c C2)} is true. Let \texttt{?b} be a behavior on norm \texttt{N}'s application grounds (via transitivity, also on the prohibition's), or that \texttt{(entails ?b B2)} is true. We show that in context \texttt{?c}, behavior \texttt{?b} is believed to be permissible by agent \texttt{A}.

Given that \texttt{N} is an obligation or discretionary norm, \texttt{N} is also a permission. Next we show that both exceptions to Inference Rule 1 are false. Given that prohibition \texttt{P} came before \texttt{N}, and no other prohibition has been stated, \texttt{(normPriorToNorm N ?proh)} is false. Thus, exception 1.1 fails. Given that \texttt{(entails B2 B1)} is true, \texttt{(uninferredSentence (entails B2 B1))} is false, and the second exception also fails. Therefore, both exceptions fail and \texttt{(uninferredSentence (permissionDefeated N B2 C2 ?b ?c ?proh))} is true. Therefore, by Inference Rule 1, \texttt{(permissible ?b ?c)} is true in \texttt{A}'s microtheory.

Therefore, in all cases, obligations and discretionary norms add exceptions to previous prohibitions (implicit and explicit) that subsume them.
\end{proof}

\begin{example}
    Karli says, ``You must share my health conditions with my children.'' At this point sharing any health conditions with her children is permissible. Later, she says, ``Do not share my medical records.'' Now sharing her medical records, even sharing her health conditions with her children, is impermissible. 
\end{example}

\begin{theorem}[Direct and Indirect Defeats of Obligations and Discretionary Norms by Prohibitions That Subsume Them]
\label{direct-ind-obl-imp}
By Inference Rule 1, obligations and discretionary norms are completely defeated by later stating a prohibition that subsumes them.
\end{theorem}
\begin{proof}Let \texttt{N} be the norm frame encoding of an obligation or discretionary norm stated by agent \texttt{A}, with behavior \texttt{B1} and context \texttt{C1}. Let \texttt{P} be the norm frame encoding of a prohibition later stated by \texttt{A}, with behavior \texttt{B2} and context \texttt{C2}. Assume \texttt{P} subsumes \texttt{N}, and thus \texttt{(entails B1 B2)} is true. Let context \texttt{?c} be a context in which both norms are active, or that \texttt{(entails ?c C1)} and \texttt{(entails ?c C2)} are true. Let \texttt{?b} be a behavior on \texttt{N}'s application grounds (via transitivity, also on the prohibition's), or that \texttt{(entails ?b B1)} is true. We show that in context \texttt{?c}, behavior \texttt{?b} is believed to be impermissible by agent \texttt{A}.

Given that \texttt{N} is an obligation or discretionary norm, \texttt{N} is also a permission. Given that \texttt{N} came before \texttt{P}, \texttt{(normPriorToNorm N P)} is true. Given that \texttt{(entails B1 B2)} is true, by exception 1.1 \texttt{(permissionDefeated N B1 C1 ?b ?c P)} is true. Therefore, \texttt{(permissible ?b ?c)} is false in \texttt{A}'s microtheory and, by Prohibitive Closure, \texttt{(impermissible ?b ?c)} is true. Therefore, prohibitions defeat earlier subsumed obligations and discretionary norms.
\end{proof}

\begin{example}
    Karli has said, ``You may share my medical records.'' She has also said, ``Do not tell my husband what prescriptions I am taking.''
    It is thus permissible to share Karli's medical records, but not her prescriptions with her husband. 
\end{example}

\begin{theorem}[Indirect Defeats of Obligations and Discretionary Norms by Strictly Subsumed Prohibitions]
\label{ind-obl-imp}
By Inference Rule 1,  obligations and discretionary norms get exceptions added by stating strictly subsumed prohibitions, regardless of order.
\end{theorem}
\begin{proof}
Let \texttt{N} be the norm frame encoding of an obligation or discretionary norm stated by an agent \texttt{A}, with behavior \texttt{B1} and context \texttt{C1}. Let \texttt{P} be the norm frame encoding of a prohibition stated by \texttt{A}, with behavior \texttt{B2} and context \texttt{C2}. Assume \texttt{N} strictly subsumes \texttt{P}. Let \texttt{?b} be a behavior on the prohibition's application grounds (via transitivity, also on \texttt{N}'s), or that \texttt{(entails ?b B2)} is true. Let \texttt{?c} be a context in which both norms are active, or that both \texttt{(entails ?c C1)} and \texttt{(entails ?c C2)} are true. We show that in context \texttt{?c}, behavior \texttt{?b} is believed to be impermissible by agent \texttt{A}.

Given that \texttt{N} is an obligation or discretionary norm, it is also a permission. However, given that \texttt{N} strictly subsumes \texttt{P}, \texttt{(entails B1 B2)} is false. Thus, exception 1.2 follows trivially, and \texttt{(permissionDefeated N B1 C1 ?b ?c P)} is true. Thus, \texttt{(permissible ?b ?c)} is false in \texttt{A}'s microtheory and, by Prohibitive Closure, \texttt{(impermissible ?b ?c)} is true. Therefore, obligations and discretionary norms get exceptions added by strictly subsumed prohibitions, regardless of temporal order.
\end{proof}

\begin{example}
    Karli has said, ``You may share my medical records.'' But she has also said, ``Do not upset my husband.'' So, it is permissible to share her medical records, but only when doing so does not also upset her husband. 
\end{example}

\begin{theorem}[Intersecting Defeats of Obligations and Discretionary Norms by Prohibitions]
\label{int-obl-imp}
By Inference Rule 1,  when an obligation or discretionary norm and a prohibition have been stated and neither subsumes the other, the obligation or discretionary norm is defeated at the intersection with the prohibition, regardless of temporal order. 
\end{theorem}
\begin{proof}
Let \texttt{N} be the norm frame encoding of an obligation or discretionary norm stated by an agent \texttt{A}, with behavior \texttt{B1} and context \texttt{C1}. Let \texttt{P} be the norm frame encoding of a prohibition stated by \texttt{A}, with behavior \texttt{B2} and context \texttt{C2}. Assume that neither norm subsumes the other. Let \texttt{?c} be a context in which both norms are active, or that both \texttt{(entails ?c C1)} and \texttt{(entails ?c C2)} are true. Let \texttt{?b} be a behavior on both norms' application grounds, or that \texttt{(entails ?b B1)} and \texttt{(entails ?b B2)} are true. We show that in context \texttt{?c}, behavior \texttt{?b} is believed to be impermissible by agent \texttt{A}.

Given that \texttt{N} is an obligation or discretionary norm, \texttt{N} is also a permission. However, given that neither norm subsumes the other, \texttt{(entails B1 B2)} is false. Thus, exception 1.2 follows trivially, and \texttt{(permissionDefeated N B1 C1 ?b ?c P)} is true. Thus, \texttt{(permissible ?b ?c)} is false in \texttt{A}'s microtheory and, by Prohibitive Closure, \texttt{(impermissible ?b ?c)} is true. Therefore, obligations and discretionary norms are defeated by prohibitions at their intersection, regardless of their temporal ordering.
\end{proof}

Theorems 1-4 demonstrate that Inference Rule 1 and 2 soundly infer agents' normative beliefs from their ongoing, possibly conflicting normative testimony. Thus, they demonstrate that the formalism meets our criteria for norm conflict resolution. Given a set $N$ of norm frame encoding of normative testimony, under Inference Rules 1 and 2 it holds that if $N \vdash permissible(b,c)$, then $N \nvdash \neg permissible(b,c)$. Next, we empirically evaluate its use in norm-guided plans.


\section{Empirical Evaluation}
 
To empirically evaluate our approach in a human-AI setting, we have developed \textit{SocialBot}. SocialBot is an AI agent built in the Companion Cognitive Architecture~\cite{forbus2017analogy} with an interface to MS Teams~\cite{MSTeams17}. Companions have been used for various AI research projects, such as knowledge extraction~\cite{Ribeiro2021} and modeling commonsense reasoning~\cite{blass2017analogical}. 

Users interact with SocialBot via NL, and SocialBot translates to semantic representations via the Companion Natural Language Understanding System (CNLU)~\cite{tomai2009ea} and pragmatic rules. We provide a screenshot of a dialogue between SocialBot and two users in Figure \ref{fig:image} \footnote{All names are changed to preserve user privacy}.

 We evaluate SocialBot in the domain of information sharing. This is a rich domain for testing norm-guided planning, as we regularly clarify who should have access to what information and expect others to adhere to these privacy norms. Thus, SocialBot supports three types of NL statements. First, users can teach it their preferences of liking or disliking something. We focus on food, drink, and academic topic preferences, while avoiding accumulating more sensitive information (e.g., political preferences), as per our IRB protocol. For example, when Karli tells Socialbot \textit{``I like AI.''}, this preference is automatically encoded as \texttt{(likesType Karli AI)} and stored in its microtheory that models Karli. Second, users can teach SocialBot privacy norms by stating normative testimony. For example, Karli states, \textit{``You must share my likes about AI.''}. As described previously, this is automatically encoded in a corresponding norm frame representation. Third, users can inquire about campus events and other users' preferences e.g., Jan asks, \textit{``What does Karli like?''}

Next, we describe how SocialBot responds to such queries with respect to users' dynamically changing privacy norms. Note that we compute entailment (Definition \ref{def:entailment}) with an algorithm similar to~\cite{olson2023mitigating} and the query freezing method described in~\cite{sensoy2012owl}.

\begin{figure}[t]
  \centering
  \setlength{\fboxsep}{0pt}
\fbox{\includegraphics[width=0.95\linewidth]{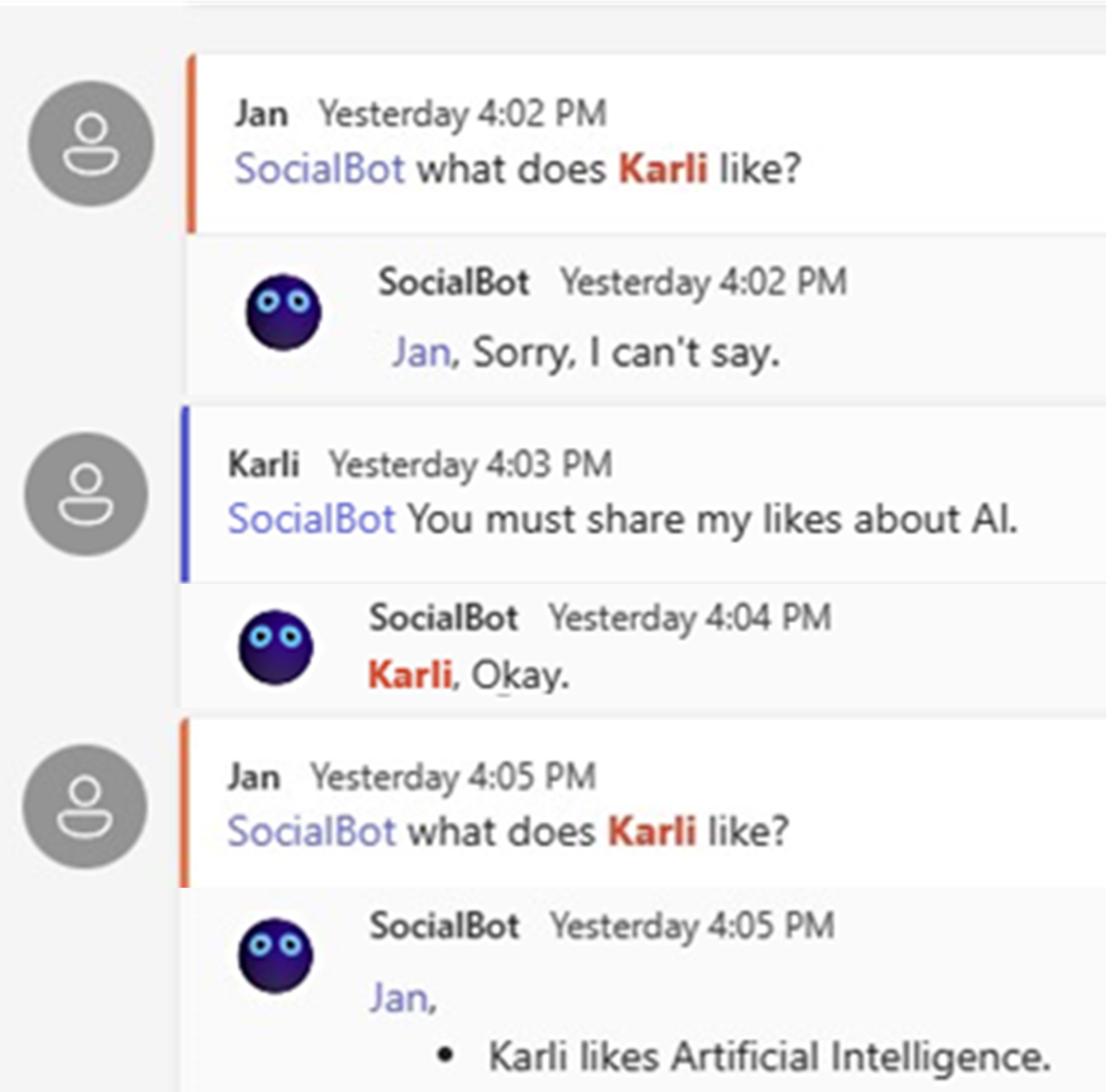}}
  \caption{SocialBot Learning and Respecting Privacy Norms.}
  \label{fig:image}
\end{figure} 

\subsubsection{Respecting Dynamically Changing Privacy Norms}
Companions use a Hierarchical Task Network (HTN) planning system~\cite{Nau1999}. As noted previously, an HTN plan is of a similar form to our abstract concept of a plan containing
\textit{action predicates}, \textit{methods} that map actions to ordered lists of sub-actions, and a set of \textit{preconditions} for methods. Thus, a \textit{ norm-guided} HTN plan is an HTN plan with a normative belief as one of its preconditions.

In this domain, SocialBot operates under a Prohibitive Closure assumption to preserve user privacy. Thus, its preconditions contain a check if sharing a user's information is permissible to that user. If this fails, the preconditions will be false, and the plan will not be executed. We provide a simplified version of one of SocialBot's norm-guided HTN plans in predicate calculus below. Note that \texttt{(ist-Information <mt> <fact>)} means that \texttt{<fact>} is true in  microtheory \texttt{<mt>}.

\texttt{(preconditionForMethod}
  
\texttt{\hspace{2mm}(and}

\texttt{\hspace{4mm}(factsInDiscourse ?d ?d-facts)}

\texttt{\hspace{4mm}(askingPreference ?dis-like ?object}

\texttt{\hspace{6mm}?owner ?asker)}

\texttt{\hspace{4mm}(ist-Information ?owner}

\texttt{\hspace{6mm}(?dis-like ?owner ?object))}

\texttt{\hspace{4mm}(ist-Information ?owner}

\texttt{\hspace{6mm}(permissible}

\texttt{\hspace{8mm}(and (isa ?share SharingPref)}

\texttt{\hspace{10mm}(object ?share ?object)}

\texttt{\hspace{10mm}(hearer ?share ?asker))}

\texttt{\hspace{8mm}?d-facts)))}

\texttt{\hspace{2mm}(methodForAction}

\texttt{\hspace{4mm}(respondToUser ?d ?s-id ?dis-like}

\texttt{\hspace{6mm}?object ?owner ?asker)}

\texttt{\hspace{4mm}(actionSequence}

\texttt{\hspace{6mm}(TheList}

\texttt{\hspace{8mm}(respond  ?d}

\texttt{\hspace{10mm}(?dis-like ?owner ?object))))))}

We illustrate with an example below by considering the interaction between Jan and SocialBot in Figure~\ref{fig:image}.

\begin{example} [How SocialBot Respects Karli's Current Privacy Norms]
When responding to Jan, SocialBot attempts to execute the action \texttt{(respondToUser ?d ?s-id
?dis-like ?object ?owner ?asker)}. It does so by first proving its method's preconditions via abduction. The first precondition bundles the facts in the current discourse into a conjunction \texttt{?d-facts} for normative reasoning.

Jan has asked what Karli likes. CNLU and pragmatic rules translates this sentence, \texttt{?s-id}, to the logical form: \texttt{(askingPreference likesType ?object Karli Jan)}, which is true in the discourse context \texttt{?d}. Thus our second precondition is true with said bindings.

With the third precondition, SocialBot queries for what Karli likes:  \texttt{(ist-Information Karli (likesType Karli ?object))}. In this case, Karli has told SocialBot she likes AI, so this query returns bindings for \texttt{?object} as \texttt{ArtificialIntelligence}. Thus, the third precondition of the method is true. When no bindings are returned, SocialBot replies, ``I don't know.''

The fourth precondition makes this a norm-guided plan. Here, with any bindings thus far, SocialBot determines if Karli current believes revealing this preference with Jan is permissible, given the facts in the current discourse context.

If this query for her normative belief fails, via our Prohibitive Closure assumption, this act is believed to be impermissible. Thus, SocialBot responds, ``Sorry I can't say.'' In this example, via Inference Rule 1 and Karli's previous normative testimony,  SocialBot can prove she believes this is permissible and thus responds with her retrieved preference.
\end{example}

\subsection{Experiment}
We hypothesize that SocialBot 1) correctly updates users' information sharing norms given their ongoing stream of normative testimony, and 2) properly answers preference queries in adherence to inferred normative beliefs. However, empirically evaluating this is challenging due to the large space of permutations of norm conflicts, preference introductions, normative testimony, and queries. We are continually collecting such real-world user interactions with SocialBot, but doing so will take quite some time. Therefore, we have constructed and tested our approach on a synthetic dataset of 1,536 NL dialogues containing all such permutations. We describe this dataset and our findings next. 

\subsection{Synthetic Dataset}
The synthetic dataset contains 1,536 dialogue cases of the form: \textit{1) norm conflict type, 2) id, 3) speaker indicator, 4 \& 5) preference introductions, 6 \& 7) normative testimony, 8) speaker indicator, 9) preference query to test, and 10) a corresponding response label we provided manually}. The conflict type, id, and speaker indicators serve as meta-data. The rest are NL sentences that simulate dialogues between SocialBot and two different users, where users change via the speaker indicators. We provide an example from the dataset below.

\begin{example}[Data Point 766] Intersecting Conflict, 766, speaker: Plato, ``I like juice.'', ``I like soda.'', ``Do not share my preferences about drinks with Socrates.'', ``You may share my preferences about juice.'', speaker: Socrates, ``What does Plato like?'', ``I can't say.''
\end{example}

\subsection{Experimental Results}
In our experiment, we first automatically input each dialogue case sentence by sentence into SocialBot. When then logged SocialBot's response to the NL query. Finally, we compared SocialBot's responses to each corresponding true label.

\textbf{SocialBot correctly responded to 100\% of the preference queries.} Given 1,536 cases in total, and 5 possible responses for each case, this yields a significantly low p-value: $\frac{1}{5}^{1,536} < .01$. Our findings are thus inconsistent with the null hypothesis that SocialBot's performance is due to correctly responding at random. Furthermore, because we tested all possible cases of normative conflicts, these findings yield support for our hypotheses that our approach 1) accurately updates users' normative beliefs given their conflicting normative testimony and 2) properly adheres to these dynamically changing normative beliefs during planning.

\section{Related Work}
Automatically resolving norm conflicts has been of recent interest in the multi-agent community~\cite{santos2017detection}. A frequently used method is to rewrite conflicting norms e.g.,~\cite{da2014normative}. Our defeasible approach via negation as failure offers a simpler update mechanism, as it does not have to maintain edits or delete norms. Our approach also formalizes concepts of deontic inheritance and thus can resolve indirect and intersecting norm conflicts, allowing agents to add exceptions to previously stated norms.

Inference Rule 1 and 2 have similarities to the rules of the Defeasible Deontic Inheritance Calculus  (DDIC) \cite{olson2024defeasibledeonticcalculusresolving}. However, our formalism builds on predicate calculus, rather than propositional logic, and is thus more expressive. Furthermore, the DDIC does not make a default assumption about normative beliefs. Most importantly, their rules are rooted in possible world semantics~\cite{kripke1963semantical}, under which our rules would be unsound. This means our formalism can make stronger inferences from permissions.

Lastly,~\cite{kollingbaum2004strategies} also provides a formalism for norm-guided planning. However, their approach contains only an implicit Permissive Assumption. Moreover, their approach operates within a constrained formal language, while we have demonstrated that our approach handles constrained natural language in a human-AI setting.

\section{Discussion of Limitations}
Here we discuss a few limitations of this work. First, our empirical evaluation was quite limited as we only considered a single action (sharing preferences) and tested on a synthetic dataset. We plan to extend our implementation to handle more sophisticated NL dialogues in future work. Second, we formalize guiding plans with a single agent's dynamically changing norms. Thus, our approach cannot guide its plans by weighing the normative beliefs of multiple agents. Third, our theoretical demonstration assumes that the agent has background knowledge necessary to compute entailment, and thus to resolve norm conflicts. Lastly, our formalism does not consider the interaction between obligations and discretionary norms.  We plan to address each of these limitations in future work.

\section{Conclusion and Future Work}

This paper presents a formalism for dynamically resolving conflicts in an agent's normative testimony to infer their normative beliefs, and then using these beliefs as guard rails on plans. We have demonstrated the efficacy of our approach through formal proofs and through an experiment with SocialBot, an AI agent built in the Companion Cognitive Architecture. Testing in the domain of information sharing, we have demonstrated that SocialBot learns users' dynamically changing privacy norms and respects these norms when sharing information with others. Thus, this paper is a contribution towards adaptive guard rails for AI agents' actions.

\bibliographystyle{named}
\bibliography{ijcai26}

\end{document}